\documentclass[12pt]{article}
\usepackage[a4paper,margin=2cm]{geometry}
\usepackage[T1]{fontenc}
\usepackage{graphicx}
\usepackage{color}
\usepackage[hang,small,bf]{caption}
\usepackage[subrefformat=parens]{subcaption}
\usepackage{bm,enumerate}
\usepackage{amsmath,amssymb,amsthm}
\usepackage{mathtools}
\usepackage{epsfig}
\usepackage{comment}
\usepackage{cite}
\usepackage{url}
\usepackage{authblk}

\newtheorem{thm}{Theorem}[section]{\bfseries}{\itshape}

\newtheorem{defi}{Definition}[section]{\bfseries}{\itshape}
\newtheorem{lem}{Lemma}[section]{\bfseries}{\itshape}
\newtheorem{cor}{Corollary}[section]{\bfseries}{\itshape}
\newtheorem{rem}{Remark}[section]{\bfseries}{\rmfamily}
\newtheorem{prop}{Proposition}[section]{\bfseries}{\itshape}
\newtheorem{asum}{Assumption}[section]{\bfseries}{\itshape}
{\bfseries}{\rmfamily}

\begin{document}
\title{An Asymptotic Equation Linking WAIC and WBIC in Singular Models}
%
%
\author[1]{Naoki Hayashi}
\author[2]{Takuro Kutsuna}
\author[3]{Sawa Takamuku}
\affil[1]{TOYOTA CENTRAL R\&D LABS., INC., Tokyo Campus, Koraku Mori Building 10F, 1-4-14 Koraku, Bunkyo-ku, Tokyo 112-0004, Japan.\protect\\
E-mail: Naoki.Hayashi.wm@mosk.tytlabs.co.jp}
\affil[2]{TOYOTA CENTRAL R\&D LABS., INC., Nagakute Campus, 41-1, Yokomichi, Nagakute, Aichi 480-1192, Japan.\protect\\
E-mail: kutsuna@mosk.tytlabs.co.jp}
\affil[3]{AISIN CORPORATION, Tokyo Research Center, Akihabara Dai Building 7F, 1-18-13 Sotokanda, Chiyoda-ku, Tokyo 101-0021, Japan.\protect\\
E-mail: sawa.takamuku@aisin.co.jp}
%
%
%
\maketitle              
\begin{abstract}
In statistical learning, models are classified as regular or singular depending on whether the mapping from parameters to probability distributions is injective.  
Most models with hierarchical structures or latent variables are singular, for which conventional criteria such as the Akaike Information Criterion and the Bayesian Information Criterion are inapplicable due to the breakdown of normal approximations for the likelihood and posterior.  
To address this, the Widely Applicable Information Criterion (WAIC) and the Widely Applicable Bayesian Information Criterion (WBIC) have been proposed.  
Since WAIC and WBIC are computed using posterior distributions at different temperature settings, separate posterior sampling is generally required.  
In this paper, we theoretically derive an asymptotic equation that links WAIC and WBIC, despite their dependence on different posteriors.  
This equation yields an asymptotically unbiased expression of WAIC in terms of the posterior distribution used for WBIC.  
The result clarifies the structural relationship between these criteria within the framework of singular learning theory, and deepens understanding of their asymptotic behavior.  
This theoretical contribution provides a foundation for future developments in the computational efficiency of model selection in singular models.

\end{abstract}
\section{\label{sec_intro}Introduction}

A statistical model is called a regular model if there exists an injective mapping from the parameter space to the set of probability distributions and the Fisher information matrix is positive definite.  
In contrast, models with hierarchical structures or latent variables--such as neural networks~\cite{Watanabe2, Aoyagi2024consideration, Aoyagi2025singular, nhayashiLinCBMRLCT}, reduced rank regression~\cite{Aoyagi1}, mixture models~\cite{Yamazaki1, Yamazaki2013comparing, SatoK2019PMM}, Markov model~\cite{Zwiernik2011asymptotic}, non-negative matrix factorization~\cite{nhayashi2, nhayashi8}, and latent Dirichlet allocation~\cite{nhayashi7, nhayashi9}--typically lack such a one-to-one map.  
These are known as singular models.  
Almost statistical models used in machine learning are singular models~\cite{Watanabe2007almost}.
In singular models, it is not possible to approximate the likelihood or posterior distribution by a normal distribution~\cite{SWatanabeBookE}.  
Moreover, it has been shown that Bayesian inference is more effective than maximum likelihood estimation or maximum a posteriori estimation in the sense that it achieves smaller generalization loss for singular models~\cite{SWatanabeBookE, SWatanabeBookMath}.

Given $n$ independent and identically random variables $D_n=(X_1, \ldots, X_n)$, we consider Bayesian inference using a statistical model $p(x|w)$.  
Here, each $X_i$ ($i=1,\ldots,n$) is assumed to follow the true distribution $q(x)$, and $w \in \mathcal{W} \subset \mathbb{R}^d$ denotes the parameter.  
Let $\beta > 0$ be the inverse temperature, $\varphi(w)$ the prior distribution, $Z_n(\beta)$ the marginal likelihood, and $\varphi^*_\beta(w|D_n)$ the (tempered) posterior distribution.  
Then, by the definition of conditional probability, the following holds:
\begin{align}
\varphi^*_\beta(w|D_n) &= \frac{\prod_{i=1}^n p(X_i|w)^\beta \varphi(w)}{Z_n(\beta)}, \\
Z_n(\beta) &= \int \prod_{i=1}^n p(X_i|w)^\beta \varphi(w) dw.
\end{align}
Define the probability distribution $p^*_\beta(x|D_n)$ as the model average with respect to the posterior distribution:
\begin{align}
p^*_\beta(x|D_n) = \int p(x|w)\varphi^*_\beta(w|D_n) dw.
\end{align}
This distribution is called the Bayesian predictive distribution.  
Statistical inference is the process of estimating the source (i.e., the true distribution) of the observed data.  
Therefore, Bayesian inference can be regarded as inferring the true distribution $q(x)$ by the Bayesian predictive distribution $p^*_\beta(x|D_n)$~\cite{SWatanabeBookMath}.

Although Bayesian inference is performed given $(p,\varphi)$ and the dataset $D_n$, it is necessary to evaluate whether $(p,\varphi)$ is appropriate with respect to the true distribution $q(x)$.  
Representative evaluation criteria include the free energy and generalization loss in generic cases \cite{SWatanabeBookMath, StatRethinkMcElreath2nd}.  
The free energy is defined as the negative log marginal likelihood:
\begin{align}
F_n(\beta) = -\frac{1}{\beta}\log Z_n(\beta).
\end{align}
The Bayesian generalization loss~$G_n$ and empirical loss~$T_n$ are defined as:
\begin{align}
G_n(\beta) &=-\int q(x)\log p^*_\beta(x|D_n)dx, \\
T_n(\beta) &= -\frac{1}{n} \sum_{i=1}^n \log p^*_\beta(X_i|D_n),
\end{align}
respectively.  
The free energy and Bayesian generalization loss represent the accuracy of inference on the data generation process via the marginal likelihood, and the predictive performance on unseen data via the Bayesian predictive distribution, respectively.
In Bayesian inference, the case $\beta=1$ is of particular importance.  
Hereafter, for all relevant quantities, the subscript and argment $\beta$ are omitted when $\beta=1$.

It has been shown that these quantities admit asymptotic expansions involving two birational invariants determined by $(q,p,\varphi)$: the real log canonical threshold (RLCT) $\lambda$ and the singular fluctuation $\nu$~\cite{SWatanabeBookE, SWatanabeBookMath}:
\begin{align}
F_n &= nL_n(w_0) + \lambda \log n + O_p(\log \log n), \label{eq_asympF} \\
\mathbb{E}[G_n] &= L(w_0)+\frac{\lambda}{n} + o\left( \frac{1}{n} \right), \label{eq_asympG} \\
\mathbb{E}[T_n] &= L(w_0)+\frac{\lambda-2\nu}{n} + o\left( \frac{1}{n} \right). \label{eq_asympT}
\end{align}
 
Here, $\mathbb{E}[\cdot]$ denotes the expectation with respect to the true distribution of the dataset $D_n$, and $L_n(w)$ and $L(w)$ are defined as
\begin{align}
L(w) &= -\int q(x) \log p(x|w) dx, \\
L_n(w) &= -\frac{1}{n} \sum_{i=1}^n \log p(X_i|w),
\end{align}
respectively.
In other words, $L(w)$ is an expected loss (also referred to as population loss \cite{Lau2023quantifying, Hoogland2024developmental}) and $L_n(w)$ is an empirical loss ($1/n$ timed negative log likelihood). Indeed, $\mathbb{E}[L_n(w)]=L(w)$ holds.
The parameter $w_0 \in \mathcal{W}$ minimizes the Kullback–Leibler divergence between the true distribution and the statistical model.  
In general, $w_0$ is not unique but continuous infinite.
Both the free energy and the Bayesian generalization loss have asymptotic expansions whose leading terms involve the RLCT.
This framework is referred to as singular learning theory~\cite{SWatanabeBookE}.

In general, even when $\beta=1$, it is often analytically and numerically difficult to compute $F_n(\beta)$ and $G_n(\beta)$.  
Moreover, model selection criteria such as the Bayesian Information Criterion (BIC)~\cite{Schwarz1978BIC} and the Akaike Information Criterion (AIC)~\cite{Akaike1974AIC, AkaikeAIC} assume regular models and are not valid for singular models~\cite{Hartigan1985, Hagiwara2002problem, Hayasaka2004asymptotic, Watanabe1995generalized}.  
Accordingly, estimators based on singular learning theory have been proposed:  
the Widely Applicable Bayesian Information Criterion (WBIC)~\cite{WatanabeBIC} for $F_n$, and the Widely Applicable Information Criterion (WAIC)~\cite{WatanabeAIC} for $G_n(\beta)$.  
Let $\mathbb{E}_w^\beta[\cdot]$ and $\mathbb{V}_w^\beta[\cdot]$ denote the expectation and variance with respect to the posterior distribution $\varphi^*_\beta(w|D_n)$, respectively.  
Then, defining WBIC as $\hat{F}_n$ and WAIC as $\hat{G}_n(\beta)$, we have:
\begin{align}
\hat{F}_n &= \mathbb{E}_w^{1/\log n} [nL_n(w)], \label{eq_WBIC} \\
\hat{G}_n(\beta) &= T_n(\beta) + \frac{\beta V_n(\beta)}{n}. \label{eq_WAIC}
\end{align}
Watanabe proved the following asymptotic relations
\begin{align}
\hat{F}_n &= nL_n(w_0) + \lambda \log n + O_p(\sqrt{\log n}), \label{eq_WBICasymp} \\
\mathbb{E}[\hat{G}_n] &= L(w_0)+\frac{\lambda}{n} + o\left( \frac{1}{n} \right), \label{eq_WAICasymp}
\end{align}
and WBIC and WAIC are asymptotic estimator of the free energy~$F_n$ and Bayesian generalization loss~$G_n(\beta)$, respectively~\cite{WatanabeBIC, WatanabeAIC}:
\begin{align}
F_n &= \hat{F}_n + O_p(\sqrt{\log n}), \label{eq_asympWBIC} \\
\mathbb{E}[G_n(\beta)] &= \mathbb{E}[\hat{G}_n(\beta)] + o(1/n). \label{eq_asympWAICbeta}
\end{align}
In particular, when $\beta=1$, the difference between the expected values of WAIC and the Bayesian generalization loss becomes even smaller \cite{WatanabeAIC}:
\begin{align}
\mathbb{E}[G_n] &= \mathbb{E}[\hat{G}_n] + o(1/n^2). \label{eq_asympWAIC}
\end{align}
Here, $V_n(\beta)$ denotes the functional variance:
\begin{align}
V_n(\beta) &= \sum_{i=1}^n \mathbb{V}_w^{\beta}[\log p(X_i|w)].
\end{align}

Noting that the Bayesian predictive distribution is given by $p^*_{\beta}(x|D_n)=\mathbb{E}_w^{\beta}[p(x|w)]$, we observe that WAIC and WBIC require expectation computations under posterior distributions with different inverse temperatures: $\mathbb{E}_w^{\beta}[\cdot]$ (particularly $\beta=1$) for WAIC, and $\mathbb{E}_w^{1/\log n}[\cdot]$ for WBIC.  
However, a direct relationship between WAIC and WBIC has not been established, making it difficult to estimate both simultaneously.
This study derives an equation that connects WAIC and WBIC.
As a consequence of this derivation, we also provide a formula for estimating the WAIC in the case $\beta=1$ with $\mathbb{E}_w^{1/\log n}[\cdot]$, instead of $\mathbb{E}_w[\cdot]$.
This work provides a theoretical foundation for future developments in model selection based on WAIC and WBIC.

The structure of this paper is as follows.  
Section~\ref{sec_estimators} introduces the underlying theory of this study (i.e., singular learning theory), with a particular focus on the equation of state. 
Section~3 presents the main theorem, which is proven in Section~4.  
Section~5 offers a discussion of the main theorem, and Section~6 concludes the paper.

\section{Singular Learning Theory and the Equation of State of Statistical Learning}
\label{sec_estimators}

As seen in Eqs.~\eqref{eq_asympF} and~\eqref{eq_asympG}, the coefficient of the leading term in the asymptotic behavior of the free energy and the generalization loss is given by the RLCT, though its definition originates elsewhere. The same applies to the singular fluctuation in Eq.~\eqref{eq_asympT}.
One equivalent definition of the RLCT involves the largest pole of the zeta function of learning theory~\cite{SWatanabeBookE}.
For details, see~\cite{SWatanabeBookE, SWatanabeBookMath}.

\begin{defi}[Real Log Canonical Threshold]
Let the average log loss function be defined by
\begin{align}
K(w) = \int q(x) \log \frac{p(x \mid w_0)}{p(x \mid w)}~dx.
\end{align}
Then, the following complex-valued function is referred to as the zeta function of learning theory:
\begin{align}
\zeta(z) = \int K(w)^z \varphi(w)~dw.
\end{align}
This function is holomorphic for $\Re(z) > 0$ and can be analytically continued to a meromorphic function over the entire complex plane. It is known that all of its poles are negative rational numbers~\cite{SWatanabeBookMath}. Denote the largest pole by $-\lambda$; then~$\lambda$ is called the real log canonical threshold (RLCT). The order of this pole is referred to as the multiplicity~$m$.
\end{defi}
The function~$K(w)$ depends only on the pair~$(q, p)$, and the zeta function~$\zeta(z)$ is determined by~$K(w)$ and the prior~$\varphi(w)$. Hence, the RLCT~$\lambda$ (and its multiplicity~$m$) is determined by the triplet~$(q, p, \varphi)$. One of the main results of singular learning theory is that the RLCT governs the asymptotic behavior of the free energy and the generalization loss, as shown in Eqs.~\eqref{eq_asympF} and~\eqref{eq_asympG}.

Since the RLCT depends on the unknown true distribution~$q(x)$, it cannot be directly computed in practice.
In fact, there are many theoretical researches to clarify the RLCT of singular models~\cite{Aoyagi1, Aoyagi2024consideration, Aoyagi2025singular, nhayashi2, nhayashi7, nhayashi8, nhayashi9, nhayashiLinCBMRLCT, SatoK2019PMM, Watanabe2, Yamazaki1, Yamazaki2013comparing, Zwiernik2011asymptotic}.
In contrast, WAIC and WBIC provide estimators for~$G_n(\beta)$ and~$F_n$ as Eqs.~\eqref{eq_asympWAICbeta} and~\eqref{eq_asympWBIC} although~$\lambda$ does not appear in the definitions of them (Eqs.~\eqref{eq_WAIC} and~\eqref{eq_WBIC})~\cite{WatanabeAIC, WatanabeBIC}.
This is because their asymptotic behaviors matche the same order of~$G_n(\beta)$ and~$F_n$ mentioned in Eqs.~\eqref{eq_WAICasymp} and~\eqref{eq_WBICasymp}.
Therefore, WAIC and WBIC allow for practical estimation of the generalization loss and free energy when $\lambda$ is unknown.

\begin{defi}[Singular Fluctuation]
Let~$V_n(\beta)$ be the functional variance. Its expectation converges as~$n \to \infty$, and we define the singular fluctuation~$\nu(\beta)$ by:
\begin{align}
\mathbb{E}[V_n(\beta)] = \frac{2}{\beta} \nu(\beta) + o(1).
\end{align}
\end{defi}

\begin{rem}
As we see later, when~$\beta \ne 1$, the asymptotic behavior of the free energy, Bayesian generalization loss, and training loss differs from the~$\beta = 1$ case. However, the RLCT~$\lambda$ is determined solely by the triplet~$(q, p, \varphi)$ and is independent of the posterior and hence of~$\beta$.
\end{rem}

Based on similar principles as WAIC and WBIC, an Equation of State for general inverse temperature~$\beta > 0$ has been derived~\cite{WatanabeEoS}\footnote{Historically, the equation of state was derived earlier than WAIC.}.
For the equation, we define Gibbs losses according to~\cite{WatanabeEoS}.

\begin{defi}[Gibbs Loss]
The Gibbs generalization loss and Gibbs training loss are defined as:
\begin{align}
G'_n(\beta) &= -\mathbb{E}_w^\beta \left[ \int q(x) \log p(x \mid w)~dx \right], \\
T'_n(\beta) &= -\mathbb{E}_w^\beta \left[ \frac{1}{n} \sum_{i=1}^n \log p(X_i \mid w) \right].
\end{align}
\end{defi}

The following result, known as the equation of state of statistical learning, relates the Bayes and Gibbs losses~\cite{WatanabeEoS}.

\begin{thm}[Equation of State of Statistical Learning~\cite{WatanabeEoS}]
\label{thm_eos}
\begin{align}
\mathbb{E}[G_n(\beta) - T_n(\beta)] &= 2\beta \mathbb{E}[T'_n(\beta) - T_n(\beta)] + o(1/n), \\
\mathbb{E}[G'_n(\beta) - T'_n(\beta)] &= 2\beta \mathbb{E}[T'_n(\beta) - T_n(\beta)] + o(1/n).
\end{align}
\end{thm}

\begin{rem}
By the definition of singular fluctuation,
\begin{align}
\frac{\beta}{2} \mathbb{E}[V_n(\beta)] = \nu(\beta) + o(1),
\end{align}
so one can substitute~$\nu(\beta)$ with~$(\beta/2) \mathbb{E}[V_n(\beta)]$ in the above formulas. This leads to the derivation of WAIC as defined in Eq.~\eqref{eq_WAIC}.
\end{rem}

In this paper, following WAIC and WBIC, we do not assume model realizability and derive a new equation linking WAIC and WBIC.

\section{Main Theorem}
\label{sec_main}

We now state the main theorem of this study. We begin by introducing the fundamental conditions (FC), which are common to singular learning theory.

\begin{asum}[Fundamental Conditions (FC)]
\label{asum_fc}
Let~$\mathbb{E}_X[\cdot] = \int q(x) [\cdot]~dx$. Assume the following:
\begin{enumerate}
  \item The parameter space~$\mathcal{W}$ of the statistical model~$p(x \mid w)$ is compact, has a non-empty interior, and a piecewise-analytic boundary.
  \item The prior~$\varphi(w)$ can be decomposed as~$\varphi(w) = \varphi_1(w) \varphi_2(w)$, where~$\varphi_1(w)$ is a non-negative analytic function and~$\varphi_2(w)$ is a strictly positive~$C^\infty$ function.
  \item Let~$L^s(q)$ be the~$L^s$ space with respect to the measure~$q(x)~dx$. For some constant~$s \geq 6$\footnote{Since~$\int q(x)~dx = 1$, the inclusion~$L^s(q) \subset L^t(q)$ holds for~$s > t \geq 1$ if~$f \in L^s(q)$.}, there exists an open set~$W' \supset \mathcal{W}$ such that~$w \mapsto f(x, w)$ is analytic as an~$L^s(q)$-valued function on~$W'$.
  \item Let~$W_\varepsilon = \{ w \in \mathcal{W} \mid K(w) \leq \varepsilon \}$. There exist constants~$\varepsilon > 0$ and~$c > 0$ such that for all~$w \in W_\varepsilon$,
  \begin{align}
  \mathbb{E}_X[f(X, w)] \geq c \mathbb{E}_X[f(X, w)^2],
  \end{align}
  where~$f(x, w) = \log (p(x \mid w_0) / p(x \mid w))$ is the log-likelihood ratio.
\end{enumerate}
\end{asum}

\begin{rem}
Imposing FC is equivalent to the assumptions used in the derivation of WAIC and WBIC~\cite{WatanabeAIC, WatanabeBIC}. Therefore, we do not assume that the true distribution is realizable by the statistical model.
\end{rem}

Under these assumptions, we state the main theorem.

\begin{thm}[Main Theorem]
\label{thm_main}
Let~$\hat{G}_n(\beta)$ denote WAIC and~$\hat{F}_n$ denote WBIC. For inverse temperature~$\beta > 0$, we have:
\begin{align}
n\mathbb{E}[\hat{G}_n(\beta)] &= \mathbb{E}[\hat{F}_n] - \lambda \left( \log n - \frac{1}{\beta} \right) + \nu\left( \frac{1}{\log n} \right) + \nu(\beta) \left( 1 - \frac{1}{\beta} \right) + o(1).
\end{align}
In particular, for~$\beta = 1$,
\begin{align}
n\mathbb{E}[\hat{G}_n] &= \mathbb{E}[\hat{F}_n] - \lambda (\log n - 1) + \nu\left( \frac{1}{\log n} \right) + o(1),
\end{align}
where~$\lambda$ is the RLCT and~$\nu(\beta)$ is the singular fluctuation.
\end{thm}

The proof of the main theorem is given in the next section. As an application, the theorem implies that both WAIC and WBIC can be approximated from the posterior distribution at inverse temperature~$1 / \log n$. In other words, once the posterior used to compute WBIC is available, WAIC can also be approximated.
To this end, we use the following proposition under FC.

\begin{prop}[Imai Estimator~\cite{Imai2019estimating}]
\label{prop_imai}
Define the estimator~$\hat{\lambda}_\mathbb{V}$ of the RLCT~$\lambda$ by:
\begin{align}
\hat{\lambda}_\mathbb{V} = \beta^2 \mathbb{V}_w^\beta [nL_n(w)],
\end{align}
where~$\beta = 1 / \log n$. Then, it holds that
\begin{align}
\hat{\lambda}_\mathbb{V} = \lambda + O_p(1 / \sqrt{\log n}).
\end{align}
This estimator is referred to as the Imai estimator~\cite{Imai2019estimating}.
\end{prop}

Using this estimator, we can treat~$\lambda$ as an estimable quantity based on data and model, leading to the following theorem.

\begin{thm}[An Asymptotic Unbiased Estimator of WAIC]
\label{thm_main_app}
Let~$\hat{G}_n$ and~$\hat{F}_n$ denote WAIC at the inverse temperature~$1$ and WBIC, respectively. Then,
\begin{align}
n\mathbb{E}[\hat{G}_n] &= \mathbb{E}\left[ \hat{F}_n - \hat{\lambda}_\mathbb{V} (\log n - 1) + \frac{1}{2 \log n} V_n\left(\frac{1}{\log n} \right) \right] + O(\sqrt{\log n}) \\
&= \mathbb{E}\left[ \mathbb{E}_w^{1/\log n}[nL_n(w)] - \frac{1}{(\log n)^2} \mathbb{V}_w^{1/\log n}[nL_n(w)] (\log n - 1) \right. \notag \\
&\quad \left. + \frac{1}{2 \log n} \sum_{i=1}^n \mathbb{V}_w^{1/\log n}[\log p(X_i \mid w)] \right] + O(\sqrt{\log n}).
\end{align}
\end{thm}

The proof is given in the next section. Note that since~$\mathbb{V}_w^\beta[\cdot] = \mathbb{E}_w^\beta[(\cdot)^2] - (\mathbb{E}_w^\beta[\cdot])^2$ holds, all the quantities inside the expectation on the right-hand side are computable from data and model, provided that the posterior expectation~$\mathbb{E}_w^{1/\log n}[\cdot]$ is available. Dividing both sides by~$n$, the right-hand side serves an asymptotically unbiased estimator of WAIC.

\section{\label{sec_proof}Proof of the Main Theorem}

\subsection{\label{subsec_pfprepare}Preliminaries}

Under the FC condition, the following asymptotic behaviors of the Bayesian and Gibbs generalization/training losses have been established~\cite{WatanabeAIC, SWatanabeBookMath}.

\begin{thm}[\cite{WatanabeAIC, SWatanabeBookMath}]\label{thm_slt_loss}
There exist random variables $\Xi_n(\beta)$ and $\rho_n(\beta)$ such that the following hold for $G_n(\beta)$, $T_n(\beta)$, $G'_n(\beta)$, and $T'_n(\beta)$:
\begin{align}
G_n(\beta) &= L(w_0) + \frac{1}{n}\left(\frac{\lambda}{\beta} + \frac{1}{2}\Xi_n(\beta) - \frac{1}{2} \rho_n(\beta) \right) + o_p \left( \frac{1}{n} \right), \\
T_n(\beta) &= L_n(w_0) + \frac{1}{n}\left(\frac{\lambda}{\beta} - \frac{1}{2}\Xi_n(\beta) - \frac{1}{2} \rho_n(\beta) \right) + o_p \left( \frac{1}{n} \right), \\
G'_n(\beta) &= L(w_0) + \frac{1}{n} \left(\frac{\lambda}{\beta} + \frac{1}{2}\Xi_n(\beta) \right) + o_p \left( \frac{1}{n} \right), \\
T'_n(\beta) &= L_n(w_0) + \frac{1}{n}\left(\frac{\lambda}{\beta} - \frac{1}{2}\Xi_n(\beta)  \right) + o_p \left( \frac{1}{n} \right).
\end{align}
Here, $\Xi_n(\beta)$ and $\rho_n(\beta)$ are determined by $(q,p,\varphi)$ and $\beta$, and satisfy the following relationships with the functional variance $V_n(\beta)$.
The random variable $\rho_n(\beta)$ converges in probability to $\rho(\beta)$, and:
\begin{align}
\mathbb{E}[\Xi_n(\beta)] &= \beta \mathbb{E}[\rho_n(\beta)] + o(1), \label{eq_singfluc_rv1} \\
\rho_n(\beta) &= \rho(\beta) + o_p(1), \label{eq_singfluc_rv2} \\
V_n(\beta) &= \rho(\beta) + o_p(1). \label{eq_singfluc_rv3}
\end{align}
\end{thm}

See~\cite{WatanabeAIC, SWatanabeBookMath} for the proofs and explicit forms of the random variables $\Xi_n(\beta)$, $\rho_n(\beta)$, and $\rho(\beta)$.

To prove Theorems~\ref{thm_main} and~\ref{thm_main_app}, we present the following lemmas.

\begin{lem}[\cite{SWatanabeBookMath}]\label{lem_singfluc}
For the random variables $\Xi_n(\beta)$ and $\rho_n(\beta)$ in Theorem~\ref{thm_slt_loss}, the functional variance $V_n(\beta)$, and the singular fluctuation $\nu(\beta)$, the following relationships hold:
\begin{align}
\mathbb{E}[\Xi_n(\beta)] &= 2\nu(\beta) + o(1), \\
\mathbb{E}[\rho_n(\beta)] &= \frac{2}{\beta}\nu(\beta) + o(1).
\end{align}
\end{lem}

\begin{proof}[{\rm Proof of Lemma~\ref{lem_singfluc}}]
This lemma is immediately obtained from the definition of the singular fluctuation and Eqs.~\eqref{eq_singfluc_rv2} and~\eqref{eq_singfluc_rv3}.
See also~\cite{SWatanabeBookMath}.
\end{proof}

\begin{lem}\label{lem_emplossWBIC}
Under the FC condition, the value $L(w_0)$ of the empirical log loss at $w_0$ can be expressed using the WBIC as:
\begin{align}
L_n(w_0) = \frac{1}{n}\hat{F}_n - \lambda \frac{\log n}{n} + \frac{1}{2n}\Xi_n\left(\frac{1}{\log n} \right) + o_p\left(\frac{1}{n}\right).
\end{align}
\end{lem}

\begin{proof}[{\rm Proof of Lemma~\ref{lem_emplossWBIC}}]
By definition, WBIC is $n$ times the Gibbs training loss with inverse temperature $1/\log n$:
\begin{align}
nT'_n(1/\log n) &= -\mathbb{E}_w^{1/\log n} \left[ \sum_{i=1}^n \log p(X_i|w) \right] \\
&= \mathbb{E}_w^{1/\log n} [nL_n(w)] \\
&= \hat{F}_n.
\end{align}
Hence, by the asymptotic expansion of the Gibbs training loss in Theorem~\ref{thm_slt_loss}, we have:
\begin{align}
\hat{F}_n = nL_n(w_0) + \lambda \log n - \frac{1}{2}\Xi_n\left(\frac{1}{\log n} \right) + o_p(1). \label{eq_asympWBICprecise}
\end{align}
Rearranging gives:
\begin{align}
L_n(w_0) = \frac{1}{n}\hat{F}_n - \lambda \frac{\log n}{n} + \frac{1}{2n}\Xi_n\left(\frac{1}{\log n} \right) + o_p\left(\frac{1}{n}\right).
\end{align}
\end{proof}

\subsection{\label{subsec_pfmain}Proof of the Main Theorem}

We now prove the main theorem using the results above.

\begin{proof}[{\rm Proof of Theorem~\ref{thm_main}}]
From Lemma~\ref{lem_emplossWBIC}, we have:
\begin{align}\label{eq_emploss_eval1}
L_n(w_0) = \frac{1}{n}\hat{F}_n - \lambda \frac{\log n}{n} + \frac{1}{2n}\Xi_n\left(\frac{1}{\log n} \right) + o_p\left(\frac{1}{n}\right).
\end{align}
Substituting Eq.~\eqref{eq_emploss_eval1} into the asymptotic expansion of the Bayes training loss from Theorem~\ref{thm_slt_loss}, we obtain:
\begin{align}
T_n(\beta) &= L_n(w_0) + \frac{1}{n}\left(\frac{\lambda}{\beta} - \frac{1}{2}\Xi_n(\beta) - \frac{1}{2} \rho_n(\beta) \right) + o_p \left( \frac{1}{n} \right) \\
&=\frac{1}{n}\hat{F}_n {-} \lambda \frac{\log n}{n} {+} \frac{1}{2n}\Xi_n\left(\frac{1}{\log n} \right) {+} \frac{1}{n}\left(\frac{\lambda}{\beta} {-} \frac{1}{2}\Xi_n(\beta) {-} \frac{1}{2} \rho_n(\beta) \right) {+} o_p \left( \frac{1}{n} \right).
\end{align}
Adding $\beta V_n(\beta)/n$ to both sides, and using the definition of WAIC, the left-hand side becomes $\hat{G}_n(\beta)$:
\begin{align}
\hat{G}_n(\beta) &= \frac{1}{n}\hat{F}_n - \lambda \frac{\log n}{n} + \frac{1}{2n}\Xi_n\left(\frac{1}{\log n} \right) + \notag \\
&\quad \frac{1}{n}\left(\frac{\lambda}{\beta} - \frac{1}{2}\Xi_n(\beta) - \frac{1}{2} \rho_n(\beta) + \beta V_n(\beta) \right) + o_p \left( \frac{1}{n} \right). \label{eq_rvWAICWBICrelation}
\end{align}
Taking expectations:
\begin{align}
\mathbb{E}[\hat{G}_n(\beta)] &= \frac{1}{n}\mathbb{E}[\hat{F}_n] - \lambda \frac{\log n}{n} + \frac{1}{2n}\mathbb{E}\left[\Xi_n\left(\frac{1}{\log n} \right)\right] \notag \\
&\quad + \frac{1}{n}\left(\frac{\lambda}{\beta} - \frac{1}{2}\mathbb{E}[\Xi_n(\beta)] - \frac{1}{2} \mathbb{E}[\rho_n(\beta)] + \beta\mathbb{E}[V_n(\beta)] \right)+ o \left( \frac{1}{n} \right).
\end{align}
Applying Lemma~\ref{lem_singfluc}:
\begin{align}
\mathbb{E}[\hat{G}_n(\beta)] &{=} \frac{1}{n}\mathbb{E}[\hat{F}_n] {-} \frac{\lambda}{n}\left(\log n {-} \frac{1}{\beta} \right) {+} \frac{1}{n} \left( \nu\left(\frac{1}{\log n} \right) {+} \nu(\beta) {-} \frac{1}{\beta}\nu(\beta) \right) {+} o \left( \frac{1}{n} \right).
\end{align}
Multiplying both sides by $n$ yields Theorem~\ref{thm_main}.  
\end{proof}

\begin{proof}[{\rm Proof of Theorem~\ref{thm_main_app}}]
From Theorem~\ref{thm_main} at $\beta=1$, we have:
\begin{align} \label{eq_maintemp1}
n\mathbb{E}[\hat{G}_n]
=\mathbb{E}[\hat{F}_n] - \lambda\left(\log n - 1 \right) + \nu\left(\frac{1}{\log n} \right) + o(1).
\end{align}
By the definition of singular fluctuation,
\begin{align}
\nu\left(\frac{1}{\log n} \right) = \frac{1}{2 \log n}\mathbb{E}\left[V_n\left(\frac{1}{\log n} \right)\right] + o(1).
\end{align}
Also, from Proposition~\ref{prop_imai}, the Imai estimator $\hat{\lambda}_{\mathbb{V}}$ satisfies:
\begin{align}
\lambda &= \hat{\lambda}_{\mathbb{V}} + O_p\left( \frac{1}{\sqrt{\log n}} \right) \\
&= \mathbb{E}[\hat{\lambda}_{\mathbb{V}} ] + O\left( \frac{1}{\sqrt{\log n}} \right).
\end{align}
Substituting these into~\eqref{eq_maintemp1} gives:
\begin{align}
n\mathbb{E}[\hat{G}_n]
=\mathbb{E}[\hat{F}_n] {-} \mathbb{E}[\hat{\lambda}_{\mathbb{V}}]\left(\log n {-} 1 \right) {+} \frac{1}{2 \log n}\mathbb{E}\left[V_n\left(\frac{1}{\log n} \right)\right] {+} O\left(\sqrt{\log n}\right).
\end{align}
The result in Theorem~\ref{thm_main_app} follows from the definitions of WBIC, the Imai estimator, and the functional variance.
\end{proof}

\section{\label{sec_discuss}Discussion}

We discuss the main theorem from three perspectives.

\subsection{On the Asymptotic Error Term in the Main Theorem}
The following proposition, easily obtained from existing results, illustrates a basic asymptotic relationship between WAIC and WBIC.

\begin{prop}\label{prop_basic_relation}
Under the same assumption of the main theorem, we have
\begin{align}
n\mathbb{E}[\hat{G}_n] = \mathbb{E}[\hat{F}_n] - \lambda(\log n -1) + (m-1)\log\log n + O(1).
\end{align}
\end{prop}

To prove it, we state the following theorem.
Although the remainder term in the WBIC–free energy relation~\eqref{eq_asympWBIC} is $O_p(\sqrt{\log n})$, a more refined asymptotic behavior has been established~\cite{WatanabeBIC}:

\begin{thm}[\cite{WatanabeBIC}]\label{thm_asympWBICprecise_ori}
Let $U_n$ be a random variable with $\mathbb{E}[U_n]=0$, and suppose it converges in distribution to a normal distribution with mean zero.  
Then the WBIC $\hat{F}_n$ satisfies:
\begin{align}
\hat{F}_n = nL_n(w_0) + \lambda \log n + U_n \sqrt{\frac{\lambda \log n}{2}} + O_p(1). \label{eq_asympWBICprecise_ori}
\end{align}
\end{thm}

\begin{proof}[{\rm Proof of Proposition~\ref{prop_basic_relation}}]
Setting $\beta=1$ and using $\mathbb{E}[L_n(w_0)]=L(w_0)$, we obtain:
\begin{align}
n\mathbb{E}[G_n] &= nL(w_0) + \lambda + o(1), \\
\mathbb{E}[F_n] &= nL(w_0) + \lambda \log n + O(\log \log n).
\end{align}
Subtracting and rearranging yields:
\begin{align}\label{eq_relationGF}
n\mathbb{E}[G_n] = \mathbb{E}[F_n] - \lambda(\log n -1) + O(\log \log n).
\end{align}
We now consider applying WAIC and WBIC to equation~\eqref{eq_relationGF}.
By equation~\eqref{eq_asympWAIC}, WAIC satisfies:
\begin{align}
\mathbb{E}[G_n] = \mathbb{E}[\hat{G}_n] + o(1/n^2),
\end{align}
which can be substituted into~\eqref{eq_relationGF}.
Since $\mathbb{E}[U_n]=0$ from Theorem~\ref{thm_asympWBICprecise_ori}, we have:
\begin{align}
\mathbb{E}[\hat{F}_n] = nL(w_0) + \lambda \log n + O_p(1),
\end{align}
and hence:
\begin{align}
\mathbb{E}[F_n] = \mathbb{E}[\hat{F}_n] + O(\log \log n).
\end{align}
Therefore:
\begin{align}
n\mathbb{E}[\hat{G}_n] = \mathbb{E}[\hat{F}_n] - \lambda(\log n -1) + O(\log \log n).
\end{align}
In fact, the~$O(\log \log n)$ term can be expressed using the multiplicity~$m$ of the RLCT~$\lambda$ \cite{SWatanabeBookE, SWatanabeBookMath}, it completes the proof.
\end{proof}

A comparison with Proposition~\ref{prop_basic_relation} shows that the main theorem provides a more refined asymptotic characterization.  
In particular, while Proposition~\ref{prop_basic_relation} yields a remainder term of order~$O(1)$, the main theorem improves this to~$o(1)$.

\subsection{Asymptotic Behavior of WBIC}  
In the proof of the main theorem, we used the asymptotic behavior of WBIC~\eqref{eq_asympWBICprecise}, derived from the fact that WBIC is the Gibbs training loss with inverse temperature $1/\log n$.
In contrast, as previously mentioned, the asymptotic form~\eqref{eq_asympWBICprecise_ori} includes a $\sqrt{\log n}$-order deviation from the free energy.  
This discrepancy arises because the asymptotic expansion of the Gibbs training loss includes a plug-in of the $n$-dependent inverse temperature.  
Since the singular fluctuation $\nu(\beta)$ is defined as an expectation over the dataset, it remains constant if $\beta$ is fixed. However, in the case of WBIC with $\beta = 1/\log n$, both the random variable $\Xi_n(1/\log n)$ and the singular fluctuation $\nu(1/\log n)$ depend on $n$.  
In other words, equation~\eqref{eq_asympWBICprecise_ori} can be interpreted as a rewritten version of~\eqref{eq_asympWBICprecise}, substituting the random variable $\Xi_n(1/\log n)$.

\begin{prop}\label{prop_rv_singfluc_asymp}
The random variable $\Xi_n(1/\log n)$ satisfies the following:
\begin{align}
-\frac{1}{2} \Xi_n\left(\frac{1}{\log n}\right) = U_n \sqrt{\frac{\lambda \log n}{2}} + O_p(1),
\end{align}
where $U_n$ is the same as in equation~\eqref{eq_asympWBICprecise_ori}.
\end{prop}

\begin{proof}[Proof of Proposition~{\rm \ref{prop_rv_singfluc_asymp}}]
From Lemma~\ref{lem_emplossWBIC}, equation~\eqref{eq_asympWBICprecise}, and Theorem~\ref{thm_asympWBICprecise_ori}, we have:
\begin{align}
&\quad nL_n(w_0)+\lambda \log n - \frac{1}{2}\Xi_n\left(\frac{1}{\log n}\right) + o_p(1) \notag \\
&= nL_n(w_0)+\lambda \log n + U_n \sqrt{\frac{\lambda \log n}{2}} + O_p(1).
\end{align}
Rearranging both sides completes the proof.
\end{proof}

While this proposition contributes to the understanding of higher-order terms in singular learning theory, neither $\Xi_n(1/\log n)$ nor $U_n$ can be directly computed from data and models. Therefore, this result does not immediately provide a way to correct WBIC.  
Imai~\cite{Imai2019overestimation} proposed a method to suppress the underestimation of WBIC by adding the singular fluctuation, based on the asymptotic behavior derived from the fact that WBIC corresponds to the Gibbs training loss with inverse temperature $1/\log n$, but the discussion remains within the expected value over the dataset.

\subsection{Applications and Limitations of the Main Theorem}

We discuss potential applications of our theoretical results and their limitations.
Theorem~\ref{thm_main_app} allows us to rewrite the integrand of~$\mathbb{E}[\cdot]$ in Theorem~\ref{thm_main} in terms of quantities that are computable from the data and model.
In particular, the right-hand side of Theorem~\ref{thm_main_app} is fully computable if one can obtain the posterior expectation with inverse temperature~$1/\log n$, i.e.,~$\mathbb{E}_w^{1/\log n}[\cdot]$, which has the same computational complexity as WBIC.
Therefore, by Theorem~\ref{thm_main_app}, once the posterior distribution for WBIC is available, WAIC can also be approximated.
This implies that the posterior distribution at inverse temperature~$1$ is no longer needed, potentially reducing the sampling cost by half when both WAIC and WBIC are to be computed.

However, the approximation accuracy of WAIC based on the posterior at inverse temperature~$1$ is asymptotically higher than that of the approximation derived from Theorem~\ref{thm_main_app}.
Specifically, the approximation error in the mean of WAIC is~$o(1/n^2)$, whereas that of Theorem~\ref{thm_main_app} is~$O(\sqrt{\log n}/n)$.

Moreover, the integrand of the expectation~$\mathbb{E}[\cdot]$ in Theorem~\ref{thm_main} can also be expressed using quantities computable from the data and model for general~$\beta$:

\begin{cor}[Corollary to Theorem~{\rm \ref{thm_main_app}}]\label{cor_main_app}
\begin{align}
n\mathbb{E}[\hat{G}_n(\beta)] &= \mathbb{E}\left[
\hat{F}_n - \hat{\lambda}_\mathbb{V}\left(\log n -\frac{1}{\beta}\right) \right. \notag \\
&\quad \left. + \frac{1}{2 \log n}V_n\left(\frac{1}{\log n} \right) + \frac{\beta}{2}V_n\left(\beta \right)\left( 1 - \frac{1}{\beta} \right)  \right]+ O\left( \sqrt{\log n} \right).
\end{align}
\end{cor}

\begin{proof}[Proof of Corollary~{\rm \ref{cor_main_app}}]
This follows by replacing the RLCT with the Imai estimator and the singular fluctuation with the functional variance, as in the proof of Theorem~\ref{thm_main_app}.
\end{proof}

However, the right-hand side of this corollary still contains quantities that depend on the posterior distribution with inverse temperature~$\beta$.
Even if such terms are moved to the left-hand side, the result essentially expresses that an asymptotically unbiased estimator for WAIC (plus some random variable) equals WBIC (plus another random variable).
Therefore, for general~$\beta$, WAIC and WBIC cannot be computed simultaneously from a single posterior distribution.
Only when~$\beta = 1$ can both WAIC and WBIC be jointly evaluated using Theorem~\ref{thm_main_app}.

In addition, Lemma~\ref{lem_emplossWBIC} allows us to evaluate the empirical log loss at~$w_0$ using WBIC and the Imai estimator:

\begin{prop}\label{prop_emplossWBICandImai}
Lemma~{\rm \ref{lem_emplossWBIC}} can be reformulated using the Imai estimator~$\hat{\lambda}_{\mathbb{V}}$ as:
\begin{align}
L_n(w_0) &= \frac{1}{n}\hat{F}_n - \hat{\lambda}_{\mathbb{V}}\frac{\log n}{n} + O_p\left(\frac{\sqrt{\log n}}{n}\right).
\end{align}
\end{prop}

\begin{proof}[Proof of Proposition~{\rm \ref{prop_emplossWBICandImai}}]
From Proposition~\ref{prop_imai}, replacing~$\lambda$ in Lemma~\ref{lem_emplossWBIC} with the Imai estimator~$\hat{\lambda}_{\mathbb{V}}$ gives:
\begin{align}
L_n(w_0) = \frac{1}{n}\hat{F}_n - \hat{\lambda}_{\mathbb{V}} \frac{\log n}{n} + \frac{1}{2n}\Xi_n\left(\frac{1}{\log n} \right) + O_p\left(\frac{\sqrt{\log n}}{n}\right).
\end{align}
Since~$\Xi_n(1/\log n)/(2n)$ is at most~$O_p(\sqrt{\log n}/n)$ from Proposition~\ref{prop_rv_singfluc_asymp}, we obtain:
\begin{align}
L_n(w_0) = \frac{1}{n}\hat{F}_n - \hat{\lambda}_{\mathbb{V}} \frac{\log n}{n} + O_p\left(\frac{\sqrt{\log n}}{n}\right).
\end{align}
\end{proof}

This result suggests a potential application to model diagnosis, such as detecting overfitting or underfitting in neural networks.  
For instance, if the empirical log loss at the learned weight~$\hat{w}$, denoted~$L_n(\hat{w})$, satisfies~$L_n(\hat{w}) < L_n(w_0)$ as described in Proposition~\ref{prop_emplossWBICandImai}, this may indicate overfitting.
Computation of WBIC and the Imai estimator in neural networks can be performed using stochastic gradient Langevin dynamics (SGLD)~\cite{Welling2011SGLD}, but the practical implementation and convergence remain open areas of investigation~\cite{Lau2023quantifying, Hoogland2024developmental}.

\section{\label{sec_conc}Conclusion}

In this study, we clarified the relationship between WAIC and WBIC, which are model selection criteria derived from singular learning theory.  
Numerical experiments on the simultaneous computation of WAIC and WBIC based on this relationship remain an important subject for future investigation.

\section*{Acknowledgments}
This preprint has no post-submission improvements or corrections.
The Version of Record of this contribution is published in the Neural Information Processing (Lecture Notes in Computer Science), ICONIP 2025 Proceedings and is available online at \url{https://doi.org/10.1007/978-981-95-4367-0_35}.

\section*{Disclosure of Interests}
The authors have no interests to be declared.
%
%
%
\bibliographystyle{plain}


\begin{thebibliography}{10}

\bibitem{AkaikeAIC}
Hirotogu Akaike.
\newblock Information theory and an extension of the maximum likelihood
  principle.
\newblock In {\em Selected papers of hirotugu akaike}, pages 199--213.
  Springer, 1998.

\bibitem{Akaike1974AIC}
Hirotugu Akaike.
\newblock A new look at the statistical model identification.
\newblock {\em IEEE transactions on automatic control}, 19(6):716--723, 1974.

\bibitem{Aoyagi2024consideration}
Miki Aoyagi.
\newblock Consideration on the learning efficiency of multiple-layered neural
  networks with linear units.
\newblock {\em Neural Networks}, 172:106132, 2024.

\bibitem{Aoyagi2025singular}
Miki Aoyagi.
\newblock Singular leaning coefficients and efficiency in learning theory.
\newblock {\em arXiv preprint arXiv:2501.12747}, 2025.

\bibitem{Aoyagi1}
Miki Aoyagi and Sumio Watanabe.
\newblock Stochastic complexities of reduced rank regression in bayesian
  estimation.
\newblock {\em Neural Networks}, 18(7):924--933, 2005.

\bibitem{Hagiwara2002problem}
Katsuyuki Hagiwara.
\newblock On the problem in model selection of neural network regression in
  overrealizable scenario.
\newblock {\em Neural Computation}, 14(8):1979--2002, 2002.

\bibitem{Hartigan1985}
JA~Hartigan.
\newblock A failure of likelihood asymptotics for normal mixtures.
\newblock In {\em Proceedings of the Berkeley Conference in Honor of J. Neyman
  and J. Kiefer, 1985}, pages 807--810, 1985.

\bibitem{Hayasaka2004asymptotic}
Taichi Hayasaka, Masashi Kitahara, and Shiro Usui.
\newblock On the asymptotic distribution of the least-squares estimators in
  unidentifiable models.
\newblock {\em Neural computation}, 16(1):99--114, 2004.

\bibitem{nhayashi8}
Naoki Hayashi.
\newblock Variational approximation error in non-negative matrix factorization.
\newblock {\em Neural Networks}, 126:65--75, 2020.

\bibitem{nhayashi9}
Naoki Hayashi.
\newblock The exact asymptotic form of bayesian generalization error in latent
  dirichlet allocation.
\newblock {\em Neural Networks}, 137:127--137, 2021.

\bibitem{nhayashiLinCBMRLCT}
Naoki Hayashi and Yoshihide Sawada.
\newblock Bayesian generalization error in linear neural networks with concept
  bottleneck structure and multitask formulation.
\newblock {\em Neurocomputing}, 638(14 July):130165, 2025.

\bibitem{nhayashi2}
Naoki Hayashi and Sumio Watanabe.
\newblock Upper bound of bayesian generalization error in non-negative matrix
  factorization.
\newblock {\em Neurocomputing}, 266C(29 November):21--28, 2017.

\bibitem{nhayashi7}
Naoki Hayashi and Sumio Watanabe.
\newblock Asymptotic bayesian generalization error in latent dirichlet
  allocation and stochastic matrix factorization.
\newblock {\em SN Computer Science}, 1(2):1--22, 2020.

\bibitem{Hoogland2024developmental}
Jesse Hoogland, George Wang, Matthew Farrugia-Roberts, Liam Carroll, Susan Wei,
  and Daniel Murfet.
\newblock The developmental landscape of in-context learning.
\newblock {\em arXiv preprint arXiv:2402.02364}, 2024.

\bibitem{Imai2019estimating}
Toru Imai.
\newblock Estimating real log canonical thresholds.
\newblock {\em arXiv preprint arXiv:1906.01341}, 2019.

\bibitem{Imai2019overestimation}
Toru Imai.
\newblock On the overestimation of widely applicable bayesian information
  criterion.
\newblock {\em arXiv preprint arXiv:1908.10572}, 2019.

\bibitem{Lau2023quantifying}
Edmund Lau, Daniel Murfet, and Susan Wei.
\newblock Quantifying degeneracy in singular models via the learning
  coefficient.
\newblock {\em arXiv preprint arXiv:2308.12108}, 2023.

\bibitem{StatRethinkMcElreath2nd}
Richard McElreath.
\newblock {\em Statistical Rethinking: A Bayesian Course with Examples in R and
  Stan}.
\newblock CRC Press, 2nd editon edition, 2020.

\bibitem{SatoK2019PMM}
Kenichiro Sato and Sumio Watanabe.
\newblock Bayesian generalization error of poisson mixture and simplex
  vandermonde matrix type singularity.
\newblock {\em arXiv preprint arXiv:1912.13289}, 2019.

\bibitem{Schwarz1978BIC}
Gideon Schwarz.
\newblock Estimating the dimension of a model.
\newblock {\em The annals of statistics}, 6(2):461--464, 1978.

\bibitem{Watanabe1995generalized}
S~Watanabe.
\newblock Generalized bayesian framework for neural networks with singular
  fisher information matrices.
\newblock In {\em Proceedings of International Symposium on Nonlinear Theory
  and Its applications, 1995}, pages 207--210, 1995.

\bibitem{Watanabe2}
Sumio Watanabe.
\newblock Algebraic geometrical methods for hierarchical learning machines.
\newblock {\em Neural Networks}, 13(4):1049--1060, 2001.

\bibitem{Watanabe2007almost}
Sumio Watanabe.
\newblock Almost all learning machines are singular.
\newblock In {\em 2007 IEEE Symposium on Foundations of Computational
  Intelligence}, pages 383--388. IEEE, 2007.

\bibitem{SWatanabeBookE}
Sumio Watanabe.
\newblock {\em Algebraix Geometry and Statistical Learning Theory}.
\newblock Cambridge University Press, 2009.

\bibitem{WatanabeAIC}
Sumio Watanabe.
\newblock Asymptotic equivalence of bayes cross validation and widely
  applicable information criterion in singular learning theory.
\newblock {\em Journal of Machine Learning Research}, 11(Dec):3571--3594, 2010.

\bibitem{WatanabeEoS}
Sumio Watanabe.
\newblock Equations of states in singular statistical estimation.
\newblock {\em Neural Networks}, 23(1):20--34, 2010.

\bibitem{WatanabeBIC}
Sumio Watanabe.
\newblock A widely applicable bayesian information criterion.
\newblock {\em Journal of Machine Learning Research}, 14(Mar):867--897, 2013.

\bibitem{SWatanabeBookMath}
Sumio Watanabe.
\newblock {\em Mathematical theory of Bayesian statistics}.
\newblock CRC Press, 2018.

\bibitem{Welling2011SGLD}
Max Welling and Yee~W Teh.
\newblock Bayesian learning via stochastic gradient langevin dynamics.
\newblock In {\em Proceedings of the 28th international conference on machine
  learning (ICML-11)}, pages 681--688. Citeseer, 2011.

\bibitem{Yamazaki2013comparing}
Keisuke Yamazaki and Daisuke Kaji.
\newblock Comparing two bayes methods based on the free energy functions in
  bernoulli mixtures.
\newblock {\em Neural Networks}, 44:36--43, 2013.

\bibitem{Yamazaki1}
Keisuke Yamazaki and Sumio Watanabe.
\newblock Singularities in mixture models and upper bounds of stochastic
  complexity.
\newblock {\em Neural Networks}, 16(7):1029--1038, 2003.

\bibitem{Zwiernik2011asymptotic}
Piotr Zwiernik.
\newblock An asymptotic behaviour of the marginal likelihood for general markov
  models.
\newblock {\em Journal of Machine Learning Research}, 12(Nov):3283--3310, 2011.

\end{thebibliography}
%




\end{document}